\documentclass[11pt]{article}
\usepackage[final]{acl}
\usepackage{times}
\usepackage{latexsym}
\usepackage[T1]{fontenc}
\usepackage[utf8]{inputenc}
\usepackage{microtype}
\usepackage{inconsolata}

\usepackage{graphicx}
\usepackage{amsmath,amssymb,amsthm}
\usepackage{booktabs}
\usepackage{multirow}
\usepackage{tikz}
\usetikzlibrary{arrows.meta,positioning,shapes.geometric,calc,fit,backgrounds}
\usepackage{pgfplots}
\pgfplotsset{compat=1.17}
\usepackage{float}
\usepackage{algorithm}
\usepackage{algpseudocode}
\usepackage{xcolor}

\makeatletter
\newenvironment{algorithmwide}[1][t]
  {\begin{figure*}[#1]\def\@captype{algorithm}}
  {\end{figure*}}
\makeatother

\definecolor{cReject}{RGB}{178,34,34}
\definecolor{cRelease}{RGB}{34,120,60}
\definecolor{cHuman}{RGB}{90,90,90}
\definecolor{cJudge}{RGB}{40,80,160}
\definecolor{cHarvest}{RGB}{150,90,20}

\newtheorem{proposition}{Proposition}

\newcommand{\boot}{\textsc{TaskBoot}}
\newcommand{\rsi}{\textsc{CertHarvest}}
\newcommand{\covr}{\mathrm{cov}}
\newcommand{\err}{\mathrm{err}}
\newcommand{\agentsuite}{\textsc{Tool}}
\newcommand{\tbench}{\textsc{Term}}
\newcommand{\arb}{\textsc{Web-A}}
\newcommand{\webm}{\textsc{Web-M}}
\newcommand{\webf}{\textsc{Web-F}}
\newcommand{\codeo}{\textsc{Code-O}}
\newcommand{\codes}{\textsc{Code-S}}

\title{Certified Selective Automation of LLM Agent Evaluation}

\author{%
Chengguang Gan\textsuperscript{1}, \enspace
Yunhao Liang\textsuperscript{2}, \enspace
Qinghao Zhang\textsuperscript{3}, \enspace
Shiwen Ni\textsuperscript{4} \\[2pt]
\textsuperscript{1}Independent Researcher \quad
\textsuperscript{2}University of Chinese Academy of Sciences \\
\textsuperscript{3}Pusan National University \quad
\textsuperscript{4}Shenzhen University of Advanced Technology \\[2pt]
\textbf{Correspondence:} \texttt{chengguangg1024@gmail.com}}

\begin{document}
\maketitle

\begin{abstract}
Evaluating LLM agents still ends with a human reading trajectories, because
automatic judges carry no guarantee on how often they are wrong. We ask the
operational question: \emph{what fraction of agent evaluation can a judge take
over, with a certificate that the error rate among auto-decided trajectories
stays below a budget $\alpha$?} Agent corpora resist the standard answer: many
agents attempt the same tasks, so trajectories arrive in correlated clusters,
and the i.i.d.\ certificates of existing selective-judging methods can
overstate what is safe: a naive certificate can claim $98\%$ automation while
its realized error exceeds the budget in $17.5\%$ of task resamples. We
introduce a task-level bootstrap certificate that is valid in every regime we
test while matching the naive certificate's coverage; finite-sample
cluster-valid alternatives certify nothing at realistic task counts. Under
this certificate, a 4B logprob judge trained with SFT and reject-weighted
GRPO certifies $0.30$--$0.59$ of evaluation on tool-use and web corpora at
$\alpha{=}0.1$, the only judge, among strongly elicited frontier models,
certifying on both headline corpora. Certified coverage is predictable before training from base
rate and discrimination alone (leave-one-corpus-out $R^2{=}0.96$). Finally,
the certificate doubles as a self-training filter: pseudo-labels harvested
inside certified regions have contamination bounded by $\alpha$ by
construction (realized $0.000$--$0.041$ across six harvests), letting a judge
enter an unseen domain at in-domain strength with zero target training labels.
\end{abstract}

\section{Introduction}
\label{sec:intro}

Whether an LLM agent actually completed its task is still, in practice, a
question answered by a person: programmatic checkers mislabel outcomes on live
websites \citep{lu2025agentrewardbench, xue2025illusion}, so benchmark authors
fall back on expert annotation of full trajectories, at minutes per
trajectory. LLM judges promise to absorb this work, and a growing line of
research measures how well they agree with humans \citep{zhuge2024agent,
pan2024autonomous}. Agreement, however, is not what a team deciding whether to
\emph{turn the judge on} needs to know. They need to know how much of the
queue the judge can take over before its mistakes exceed what they can
tolerate, and that number has to hold on the next batch, not just the last.

We study this question as a certification problem: given a judge, an error
budget $\alpha$, and a confidence level $1-\delta$, find the largest fraction
of trajectories that can be \emph{auto-decided} (rejected as failures or
released as successes) such that with probability $1-\delta$ the error rate
among auto-decided trajectories is at most $\alpha$. This \emph{certified
coverage} is the fraction of human evaluation work provably removed
(Figure~\ref{fig:example}). Certifying a fixed decision rule at a risk level
is well-trodden ground \citep{bates2021distribution, angelopoulos2025learn}, recently
applied to LLM judges of chatbot responses \citep{jung2025trust, badshah2026scope}.
Agent evaluation breaks the key assumption all of it rests on.

The break is structural. Agent corpora are built by running \emph{many} agents,
or many rollouts, against the \emph{same} tasks: five attempts at the same
booking flow succeed or fail together, because difficulty lives mostly in the
task. Trajectories therefore arrive in correlated clusters, violating the
exchangeability that i.i.d.\ certificates assume. The violation is not
cosmetic. On our most heavily clustered corpus (intra-task correlation
$\rho=0.80$), the standard i.i.d.\ certificate happily certifies $98\%$
automation; a task-resampling audit shows its realized error exceeding the
budget in $17.5\%$ of resamples, against a promised $5\%$
(Section~\ref{sec:setup}). The textbook fixes fail in the opposite direction:
a design-effect correction and the finite-sample cluster-valid constructions
(one trajectory per task, task-level Learn-then-Test) certify essentially
\emph{nothing} at realistic task counts (Section~\ref{sec:cert}). Between an
invalid certificate and a vacuous one, neither is usable.

Our first contribution is a certificate that is both valid and usable:
\boot{}, a task-level bootstrap test applied over a pre-declared threshold
grid with Bonferroni accounting. In a synthetic study with known ground truth
it violates its guarantee in at most $1\%$ of trials, within budget, down to
20 task clusters, while matching the naive certificate's coverage almost
exactly; on real corpora it certifies $0.30$--$0.84$ where every finite-sample
cluster-valid alternative certifies $0$. All certificates in this paper,
including those of the frontier judges we compare against, are computed by
this one procedure with thresholds calibrated on held-out tasks and audited
out-of-sample.

With the certificate fixed, the judge becomes the object of study, and three
findings emerge. First, a 4B logprob judge, fine-tuned then trained with a
reject-weighted GRPO objective, certifies $0.297$ on tool-use and $0.585$ on
web at $\alpha{=}0.1$, and is the only judge among strongly elicited frontier
models that certifies on both headline corpora; the one frontier configuration
that beats it costs roughly $100\times$ more per decision, and no monotone
recalibration can rescue the others, because the certificate depends on scores
only through their ranks (Section~\ref{sec:judges}). Second, certified
coverage is predictable before any training: a two-parameter relation on base
rate and base discrimination explains held-out coverage with
leave-one-corpus-out $R^2=0.96$, and three gates predict when reinforcement
learning adds coverage over supervised fine-tuning
(Section~\ref{sec:model}). Third, the certificate does double duty as a
self-training filter: pseudo-labels harvested inside certified regions have
contamination bounded by $\alpha$ by construction (realized
$0.000$--$0.041$ across six harvests), letting a judge enter an unseen domain
with zero target-domain training labels at in-domain strength, and lifting our
largest tool-use corpus beyond its best supervised judge
(Section~\ref{sec:rsi}).

\begin{figure*}[t]
\centering
\includegraphics[width=0.82\textwidth]{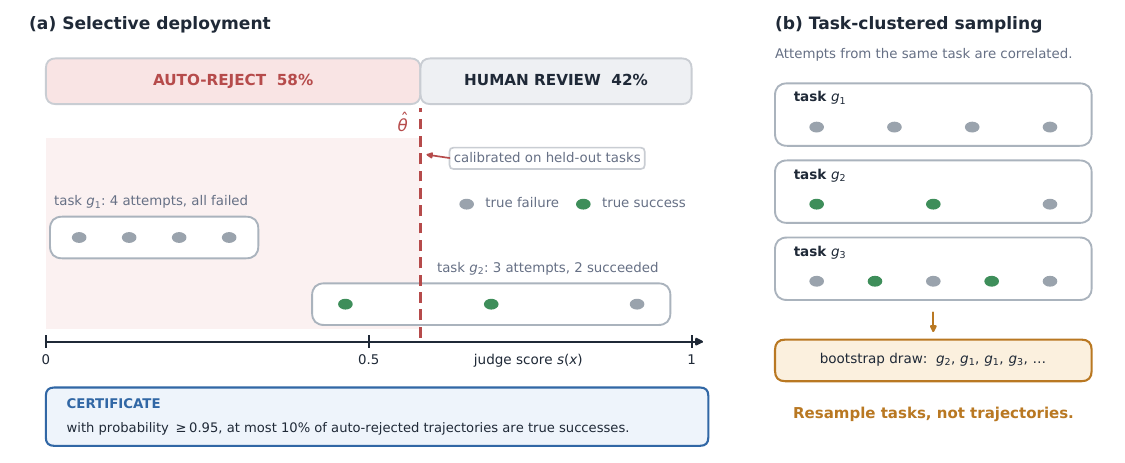}
\vspace{-5mm}
\caption{\textbf{Certified selective automation, on one corpus.} (a)~A
calibrated threshold auto-rejects low-scoring trajectories; the certificate
bounds the error rate \emph{among auto-decided trajectories}. (b)~Attempts at
one task are correlated: resample tasks, not trajectories.}
\label{fig:example}
\end{figure*}

The empirical scope is seven corpora: tool-use, terminal, three web corpora
(including 1{,}302 expert-annotated trajectories from five live benchmarks),
and code repair. The protocol is pre-registered: task-level splits fixed by
one seed, hyperparameters selected on calibration data only, one test read
per configuration. Corpora where certification fails stay in the main text:
judges are blind where ground truth never appears in the trajectory,
saturated corpora leave RL nothing to add, and thirteen test tasks cannot
support a certificate at all.

\section{Related work}
\label{sec:related}

\paragraph{Judging agent trajectories.}
That automatic evaluation of agents is unreliable is well documented:
rule-based evaluators under-report success on live websites, LLM judges
over-report it \citep{lu2025agentrewardbench, xue2025illusion}. Remedies train
or prompt better evaluators (trajectory judges \citep{pan2024autonomous},
agentic judges \citep{zhuge2024agent}, process reward models
\citep{NEURIPS2025_5bac3621}) and report agreement with human labels. We take
the step none of these take: attaching a statistical guarantee to the judge's
deployment, so that ``how good is the judge'' becomes ``how much work can it
provably absorb.'' \citet{lu2025agentrewardbench} supplies one of our corpora.

\paragraph{Selective prediction with guarantees.}
The machinery for risk-controlled selective decisions is mature: selective
classification \citep{geifman2017selective}, risk-controlling prediction sets
\citep{bates2021distribution}, Learn-then-Test \citep{angelopoulos2025learn},
conformal risk control \citep{angelopoulos2024conformal}. Two recent papers
instantiate it for LLM judges: \citet{jung2025trust} certify human-agreement
of a selective judge of chatbot responses, and \citet{badshah2026scope} wrap
a pairwise judge in conformal risk control. Both assume exchangeable
instances, and Section~\ref{sec:cert} shows that assumption is load-bearing:
the i.i.d.\ certificate becomes anti-conservative exactly in the multi-rollout
regime modern agent evaluation uses. Cluster-aware conformal methods exist
for hierarchical data \citep{dunn2023distribution, barber2023conformal}, but
their finite-sample constructions collapse to zero coverage at realistic task
counts; the working point between invalid and vacuous is, to our knowledge,
unoccupied. Prediction-powered evaluation \citep{boyeau2024autoeval} certifies
\emph{aggregate} metrics from judge labels plus few human labels; we certify
per-trajectory \emph{decisions}, a different estimand.

\paragraph{Training judges, and judges that train themselves.}
RL is now standard for improving judge accuracy
\citep{whitehouse2026j1, xu2026j4r, wang2024self}, with GRPO
\citep{shao2024deepseekmathpushinglimitsmathematical} the common optimizer; our GRPO stage differs in
objective, not machinery: an asymmetric reward aimed at certified coverage,
whose direction we ablate. Self-improving models
\citep{yuan2024self, wu2025meta, zhang2026darwin,
zhao2026absolute} gate their own training data with heuristics, and the
failure mode is documented: pseudo-label accuracy in R-Zero decays from
$79\%$ to $63\%$ over iterations \citep{huang2026r}. Conformal filters
appear in classical semi-supervised learning \citep{lienen2023conformal,
tanha2022cpssds} and in one-shot auto-labeling with FDR control
\citep{huang2025model}, but no prior system closes the loop in
which a \emph{deployment certificate} gates harvesting, training, and
re-certification; Section~\ref{sec:rsi} builds that loop and measures where
its guarantee ends. Finally, fine-tuned small judges are known to match large
ones in-domain \citep{kim2024prometheus, zhu2025judgelm} and frontier
logprobs to be miscalibrated after RLHF \citep{tian2023just,
kadavath2022language}; our frontier comparison sharpens both observations
with the certificate as the yardstick, and Proposition~\ref{prop:rank} shows
recalibration cannot change the verdict.

\section{Certified selective automation, and why clustering breaks it}
\label{sec:setup}

\paragraph{Setup.}
A corpus is a set of trajectories $x_{gj}$, where $g \in \{1,\dots,G\}$ indexes
\emph{tasks} and $j \in \{1,\dots,m_g\}$ indexes attempts at task $g$ by
different agents or rollouts. Each trajectory carries a binary outcome
$y_{gj}\in\{0,1\}$ ($1$: the agent truly completed the task), obtained from
expert annotation or a trusted oracle. A judge maps a trajectory to a score
$s(x)\in[0,1]$, its estimate of $\Pr(y{=}1\,|\,x)$. A \emph{selective
automation rule} is a pair of thresholds $(\theta_{\mathrm{rej}},
\theta_{\mathrm{rel}})$: trajectories with $s(x)\le\theta_{\mathrm{rej}}$ are
auto-rejected, those with $s(x)\ge\theta_{\mathrm{rel}}$ are auto-released, and
the rest go to a human. The two automated decisions carry different risks and
are certified separately:
\begin{align}
\err_{\mathrm{rej}}(\theta) &= \Pr\!\big(y{=}1 \,\big|\, s(x)\le\theta\big),
\nonumber\\
\err_{\mathrm{rel}}(\theta) &= \Pr\!\big(y{=}0 \,\big|\, s(x)\ge\theta\big),
\label{eq:errors}
\end{align}
i.e., the rate of \emph{discarded successes} on the reject side and of
\emph{failures shipped as successes} on the release side. For a side with error
functional $\err$ and coverage $\covr(\theta)=\Pr(\text{auto-decided at }
\theta)$, the object we report is
\begin{align}
\covr^\star(\alpha,\delta) \;=\; \max_{\hat\theta}\;&\covr(\hat\theta)
\nonumber\\
\text{s.t.}\;\;
\Pr\big(\err(\hat\theta) > \alpha\big) &\le \delta,
\label{eq:certcov}
\end{align}
with the probability taken under the full selection procedure: the largest
certifiable fraction of evaluation work removed at error budget $\alpha$ with
confidence $1-\delta$. Throughout, $\delta=0.05$ and
$\alpha\in\{0.1, 0.2\}$; the reject side is the main object because it is where
our corpora admit non-trivial certificates, and the release side is reported
where it is non-zero.

\begin{figure*}[t]
\centering
\includegraphics[width=0.82\textwidth]{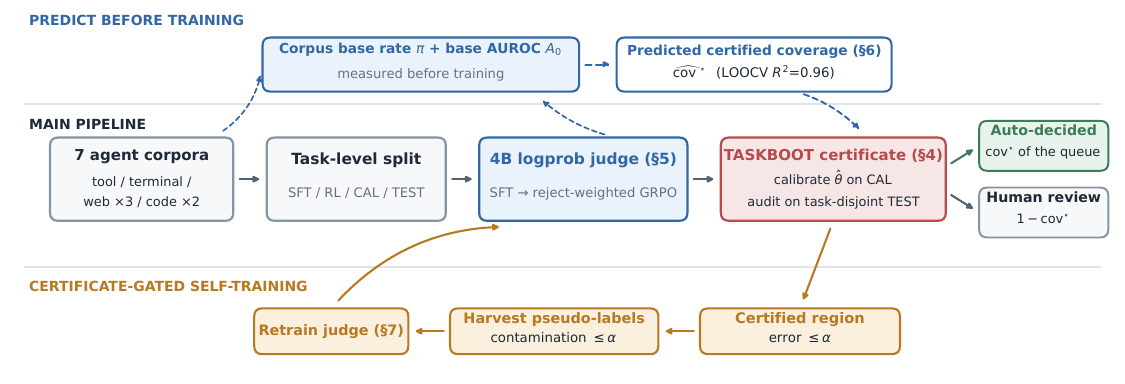}
\vspace{-5mm}
\caption{\textbf{Overview.} One certificate fixes the deployment operating
point (\S\ref{sec:cert}), is predictable from corpus structure before training
(\S\ref{sec:model}), and gates a self-training loop whose contamination is
bounded by the same $\alpha$ (\S\ref{sec:rsi}).}
\label{fig:overview}
\end{figure*}

\paragraph{Corpora and protocol.}
We build seven corpora spanning the domains agents are actually evaluated on
(Table~\ref{tab:corpora}; details in Appendix~\ref{app:corpora}): tool-use
dialogues ({\agentsuite}, from $\tau^2$-bench), terminal sessions ({\tbench}),
three web corpora, {\arb} (1{,}302 expert-annotated trajectories over five
live web benchmarks) and two MiniWoB++ corpora with weak and frontier
agents, and two code-repair corpora. Every corpus is split \emph{by task}
into SFT/RL/CAL/TEST with one declared seed; every selected hyperparameter
uses CAL only, and TEST is read once per configuration. Trajectories are
rendered as text with no reward or oracle signal included: the judge must
infer the outcome from what the agent saw and did.

\begin{table*}[t]
\centering
\caption{Corpora. $n$: test trajectories; $G$: test tasks; $\pi$: success
rate; $\rho$: intra-task correlation of outcomes; $\mathrm{DEFF} =
1+(\tilde m-1)\rho$ with $\tilde m$ the size-weighted mean cluster size;
$n_{\mathrm{eff}} = n/\mathrm{DEFF}$.}
\label{tab:corpora}
\small
\begin{tabular}{@{}llrrrrrr@{}}
\toprule
corpus & domain & $n$ & $G$ & $\pi$ & $\rho$ & DEFF & $n_{\mathrm{eff}}$ \\
\midrule
\agentsuite{} ($\tau^2$) & tool-use dialogues & 1280 & 84 & .48 & .25 & 5.0 & 257 \\
\tbench{} & terminal sessions & 520 & 13 & .34 & .43 & 17.6 & 30 \\
\arb{} (ARB) & live web, expert labels & 200 & 68 & .26 & .49 & 2.2 & 91 \\
\textsc{Web-M} & MiniWoB++, 4B agents & 512 & 32 & .16 & .80 & 13.0 & 39 \\
\textsc{Web-F} & MiniWoB++, frontier agents & 256 & 32 & .26 & .81 & 6.7 & 38 \\
\textsc{Code-O} & code repair (OpenHands) & 899 & 382 & .45 & .64 & 2.0 & 459 \\
\textsc{Code-S} & code repair (multilingual) & 21 & 7 & .24 & .00 & 1.0 & 21 \\
\bottomrule
\end{tabular}
\end{table*}

\paragraph{Clustering is large, and it breaks the i.i.d.\ certificate.}
Outcomes within a task are strongly correlated: $\rho$ ranges from $0.25$ to
$0.81$ on our corpora, and design effects reach $17.6$: each terminal task
carries the information of roughly two independent trajectories, not forty.
The standard certificate ignores this: the construction shared by
\citet{jung2025trust} and \citet{badshah2026scope} instantiates here as certifying
$\theta$ whenever the Clopper--Pearson upper bound at level $\delta/|\Theta|$
on the auto-decided error, computed \emph{as if trajectories were
independent}, is at most $\alpha$. We audited it by resampling \emph{tasks}
and recomputing the realized error at its selected threshold. On \webm{}
($\rho=0.80$, $G=32$, eight rollouts per task; the pass@$k$ evaluation
pattern) it selects a threshold covering $98\%$ of trajectories at
$\alpha=0.2$, and its realized error exceeds the budget in $17.5\%$ of task
resamples, three and a half times the promised $\delta=5\%$ ($6.9\%$ at
$\alpha=0.1$). On mildly clustered corpora the audit passes ($\le 3.2\%$;
Appendix~\ref{app:audits}): the failure is precisely the high-$\rho$,
few-task, many-rollout regime that modern agent evaluation produces, and a
usable certificate has to survive it.

\section{\boot{}: a task-level bootstrap certificate}
\label{sec:cert}

The certificate must respect two constraints that pull in opposite directions:
validity under task clustering, and non-vacuity at $G\in[10,100]$ tasks, which
is what real corpora provide. \boot{} resolves the tension by testing each
candidate threshold against the task-resampling distribution of its own error.

\paragraph{Procedure.}
Fix a side (say reject), a pre-declared grid $\Theta$ of $|\Theta|=40$ score
quantiles, and calibration data $\{(s_{gj}, y_{gj})\}$ grouped into $G$ tasks.
For a threshold $\theta$, let
\begin{equation}
\label{eq:booterr}
\widehat{\err}^{(b)}(\theta) \;=\;
\frac{\sum_{g\in\mathcal{B}_b}\, k_g(\theta)}{\sum_{g\in\mathcal{B}_b}\, n_g(\theta)},
\qquad b = 1,\dots,B,
\end{equation}
where $\mathcal{B}_b$ is a multiset of $G$ tasks drawn with replacement,
$n_g(\theta)=\#\{j: s_{gj}\le\theta\}$ and $k_g(\theta)=\#\{j: s_{gj}\le\theta,\,
y_{gj}=1\}$ are the per-task covered and erroneous counts. Threshold $\theta$
is \emph{certified} if the upper $\big(1-\delta/|\Theta|\big)$ empirical
quantile of $\{\widehat{\err}^{(b)}(\theta)\}_{b=1}^{B}$ is at most $\alpha$;
the procedure returns the certified threshold of maximal coverage
(Algorithm~\ref{alg:boot}), with $\delta/|\Theta|$ a Bonferroni correction
over the grid. This is an approximation, not a finite-sample theorem: the
bootstrap quantile consistently estimates the clustered error ratio's sampling
distribution as $G$ grows \citep{field2007bootstrapping}, and the question
that matters (does the approximation already hold at the $G$ we have?) we
answer by simulation with known ground truth and by out-of-sample audits on
every real corpus.

\begin{algorithmwide}[t]
\caption{\boot{} (one side; reject shown)}
\label{alg:boot}
\begin{algorithmic}[1]
\Require task-grouped calibration scores $\{s_{gj},y_{gj}\}$; $\alpha$, $\delta$, $|\Theta|$, $B$
\State $\Theta \gets$ empirical score quantiles at levels
$\mathrm{linspace}(0.02, 0.98, |\Theta|)$
\For{$\theta \in \Theta$}
  \State compute per-task counts $n_g(\theta), k_g(\theta)$
  \For{$b = 1,\dots,B$}
     \State draw $G$ tasks with replacement; compute
     $\widehat{\err}^{(b)}(\theta)$ by Eq.~(\ref{eq:booterr})
  \EndFor
  \State $q(\theta) \gets$ empirical $\big(1-\delta/|\Theta|\big)$-quantile of
  $\{\widehat{\err}^{(b)}(\theta)\}$
\EndFor
\State \Return $\hat\theta = \arg\max\{\covr(\theta) : \theta\in\Theta,\;
q(\theta)\le\alpha\}$ \Comment{$\emptyset \Rightarrow$ certify nothing}
\end{algorithmic}
\end{algorithmwide}

\begin{figure*}[t]
\centering
\includegraphics[width=0.82\textwidth]{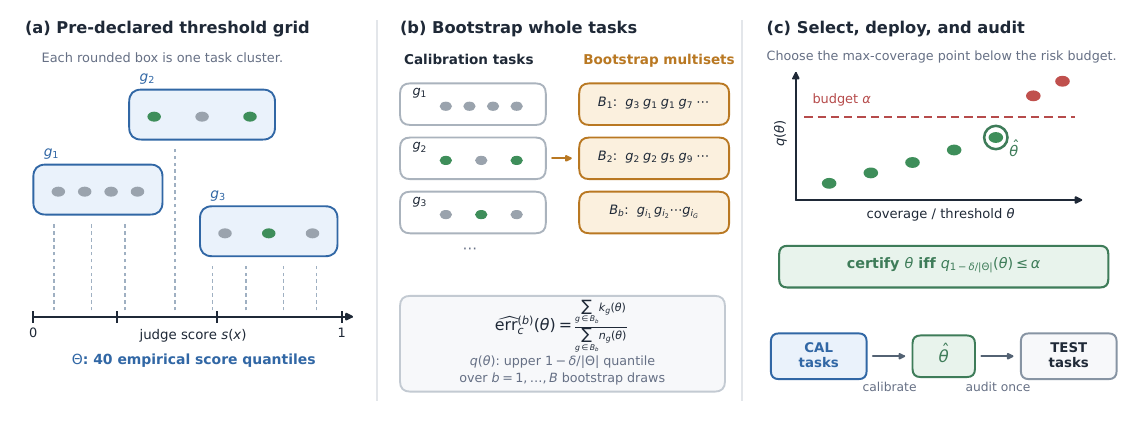}
\vspace{-5mm}
\caption{\textbf{\boot{} in one picture:} a pre-declared quantile grid (a), a
task-resampled test of each threshold's selective error (b), max-coverage
selection with calibrate-on-CAL / audit-on-TEST deployment (c).}
\label{fig:method}
\end{figure*}

\paragraph{Validity with known ground truth.}
We simulate clustered judge scores from a task-effect model: task effects
$u_g\sim\mathcal N(0,\tau^2)$ induce label correlation
$\rho\in\{0.1,0.5,0.8\}$ at $G\in\{20,50,100,500\}$ tasks; each procedure
selects a threshold on a simulated calibration draw, and we measure that
threshold's \emph{true} selective error on a fresh population of
$1.5\times10^5$ trajectories, 300 trials per cell (full protocol and grid in
Appendix~\ref{app:sim}). \boot{} violates its guarantee in at most $1\%$ of
trials in every cell, including $G{=}20$, while matching naive
Clopper--Pearson's coverage to within a point. The finite-sample cluster-valid
alternatives are valid and useless: one trajectory per task, or a Hoeffding
bound on task-mean losses, certifies zero coverage until $G$ reaches the
hundreds, and the design-effect correction $n\mapsto n/\mathrm{DEFF}$
\citep{kish1965survey} is similarly blunt at high $\rho$. Three adversarial
designs (matched to \webm{}'s parameters, size--outcome correlation,
heavy-tailed cluster sizes) did not break the naive certificate at population
level (its Bonferroni slack absorbs a lot), so the evidence is stated
precisely: naive certification fails the task-resampling audit on real
high-$\rho$ data ($17.5\%$, Section~\ref{sec:setup}), and we could not
construct any regime where \boot{} fails. Uniformly safe, at no cost in
coverage.

\paragraph{On real corpora, the alternatives certify nothing.}
Running all five procedures on the real test scores
(Appendix~\ref{app:baselines}): one-per-task CP and task-mean Hoeffding
certify exactly $0$ on \emph{all seven corpora} ($G$ from 7 to 382), and
DEFF-corrected CP certifies $0$ everywhere at $\alpha=0.1$, while \boot{}
certifies $0.30$--$0.84$ wherever a usable judge exists. Between a certificate
that can overpromise ($98\%$ on \webm{}) and certificates that always promise
nothing, \boot{} is the only occupant of the working regime.

\paragraph{Out-of-sample audits.}
Because the guarantee is asymptotic in $G$, every deployment in this paper is
audited: $\hat\theta$ is calibrated on CAL tasks, and the realized selective
error is measured once on task-disjoint TEST. All audits pass: e.g., realized
reject error at $\alpha=0.1$ is $0.027$--$0.037$ on \arb{} and
$0.011$--$0.012$ on \agentsuite{} (all listed in Appendix~\ref{app:audits}).
So every certificate below is backed twice: by simulation where truth is
known, and by held-out audits where it is not.

\begin{figure}[t]
\centering
\begin{tikzpicture}
\begin{axis}[
  width=1.02\columnwidth, height=0.56\columnwidth,
  xlabel={$(1-\pi)(2A_0-1)$ \, (measured before training)},
  ylabel={best certified coverage @ $\alpha{=}0.1$},
  xmin=-0.02, xmax=0.9, ymin=-0.04, ymax=0.9,
  grid=both, grid style={black!10},
  tick label style={font=\scriptsize},
  label style={font=\scriptsize},
  legend style={font=\tiny, at={(0.5,-0.34)}, anchor=north,
                legend columns=2, draw=black!20, /tikz/every even column/.append style={column sep=6pt}},
  legend cell align=left,
]
\addplot[domain=0:0.87, thick, black!60] {1.075*x - 0.048};
\addlegendentry{fit: $1.08x - 0.05$ (LOOCV $R^2{=}0.96$)}
\addplot[only marks, mark=*, mark size=2.4pt, cRelease]
  coordinates {(0.308,0.321) (0.588,0.585)};
\addlegendentry{trainable (RL gains)}
\addplot[only marks, mark=square*, mark size=2.2pt, cJudge]
  coordinates {(0.809,0.836) (0.717,0.758)};
\addlegendentry{saturated (SFT/base ceiling)}
\addplot[only marks, mark=triangle*, mark size=2.8pt, cReject]
  coordinates {(0.034,0.0)};
\addlegendentry{judge-blind}
\addplot[only marks, mark=diamond*, mark size=2.6pt, black!45]
  coordinates {(0.482,0.379) (0.629,0.619)};
\addlegendentry{$G<20$}
\node[font=\tiny, anchor=west] at (axis cs:0.045,0.015) {\codeo};
\node[font=\tiny, anchor=south] at (axis cs:0.308,0.34) {\agentsuite};
\node[font=\tiny, anchor=north] at (axis cs:0.482,0.365) {\tbench};
\node[font=\tiny, anchor=south] at (axis cs:0.588,0.60) {\arb};
\node[font=\tiny, anchor=north west] at (axis cs:0.635,0.615) {\codes};
\node[font=\tiny, anchor=south east] at (axis cs:0.809,0.845) {\webm};
\node[font=\tiny, anchor=north west] at (axis cs:0.717,0.75) {\webf};
\end{axis}
\end{tikzpicture}
\vspace{-2mm}
\caption{\textbf{Certifiability is measurable ex-ante.} Best certified reject
coverage against the pre-training index $(1-\pi)(2A_0-1)$. The largest
residual (\tbench{}) is the smallest-$n_{\mathrm{eff}}$ corpus, where
finite-sample slack in the certificate binds before judge quality does.}
\label{fig:model}
\end{figure}
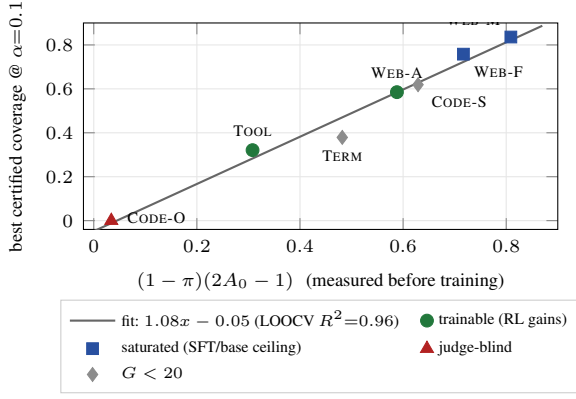

\section{Judges under the certificate}
\label{sec:judges}

\paragraph{A small logprob judge.}
The judge is a 4B-parameter instruction model prompted as a skeptical auditor:
it reads the rendered trajectory and must answer with a single verdict token.
Its score is the renormalized next-token probability
\begin{equation}
\label{eq:score}
s(x) \;=\; \frac{p_{\mathrm{LM}}(\texttt{SUCCESS}\mid x)}
{p_{\mathrm{LM}}(\texttt{SUCCESS}\mid x) + p_{\mathrm{LM}}(\texttt{FAIL}\mid x)},
\end{equation}
read in a single forward pass with reasoning disabled: every trajectory parses
by construction and the score is continuous, both properties the certificate
needs. Training has two stages. \emph{SFT} fine-tunes on verdict-labeled
trajectories from the task-disjoint SFT split (Appendix~\ref{app:training});
\emph{GRPO} \citep{shao2024deepseekmathpushinglimitsmathematical} then
optimizes an asymmetric reward on the RL split: for sampled verdict $v$
against label $y$,
\begin{align}
r(v, y) \;=\; \mathbf{1}[v = y]
\;&-\; \lambda\,\mathbf{1}[v{=}\texttt{S}, y{=}0] \nonumber\\
&-\; \mu\,\mathbf{1}[v{=}\texttt{F}, y{=}1],
\label{eq:reward}
\end{align}
so $\mu>\lambda$ penalizes discarding successes (the reject-side error) more
than shipping failures, pushing probability mass out of the reject tail
exactly where the certificate reads it. The strength $\mu\in\{1,3,5\}$ (at
$\lambda{=}1$) is selected \emph{on CAL by certified coverage}, never on TEST;
the test-optimal arm is reported once, labeled as an oracle.

\begin{table*}[t]
\centering
\caption{Main result: certified reject coverage at $\alpha{=}0.1$,
$\delta{=}0.05$ (\boot{}, threshold calibrated on CAL), and test AUROC.
GRPO arm selected on CAL. Grey: $G$ below the simulation-validated regime
($G\ge 20$), certificate reported for completeness only. Oracle-arm and
$\alpha{=}0.2$ results in Appendix~\ref{app:fulltables}.}
\label{tab:main}
\small
\setlength{\tabcolsep}{4.2pt}
\begin{tabular}{@{}lcccccccl@{}}
\toprule
& \multicolumn{3}{c}{AUROC} & \multicolumn{3}{c}{certified coverage @ $\alpha{=}0.1$} & & \\
\cmidrule(lr){2-4}\cmidrule(lr){5-7}
corpus & base & SFT & GRPO & base & SFT & GRPO & $\Delta_{\mathrm{RL}}$ & regime \\
\midrule
\agentsuite{} & .798 & .899 & .903 & .000 & .293 & \textbf{.297} & $+.004$ & trainable \\
\arb{}  & .897 & .925 & .922 & .510 & .560 & \textbf{.585} & $+.025$ & trainable \\
\webm{} & .983 & .976 & .977 & \textbf{.832} & .789 & .836 & $+.047$ & saturated \\
\webf{} & .983 & .991 & .994 & .707 & \textbf{.758} & .758 & $\pm.000$ & saturated \\
\codeo{} & .531 & .645 & .530 & .000 & .000 & .000 & --- & judge-blind \\
\color{black!50}\tbench{} & \color{black!50}.864 & \color{black!50}.871 &
\color{black!50}.873 & \color{black!50}.290 & \color{black!50}.379 &
\color{black!50}.369 & \color{black!50}$-.010$ & \color{black!50}$G{=}13$ \\
\color{black!50}\codes{} & \color{black!50}.913 & \color{black!50}.800 &
\color{black!50}--- & \color{black!50}.619 & \color{black!50}.619 &
\color{black!50}--- & \color{black!50}--- & \color{black!50}$G{=}7$ \\
\bottomrule
\end{tabular}
\end{table*}

\paragraph{Main result.}
Table~\ref{tab:main} is the paper's central table. On the two trainable
corpora the ordering $\text{GRPO} > \text{SFT} > \text{base}$ holds under the
clean selection protocol: \agentsuite{} goes $0 \to .293 \to .297$ (the
test-oracle arm reaches $.321$; every arm, including the worst, stays at or
above SFT), and \arb{} goes $.510 \to .560 \to .585$, where CAL independently
selects the arm the oracle would. The margins are modest; what makes them
meaningful is that each is a certified, audited increment in evaluation work
removed under a pre-registered protocol. The remaining corpora fail for three
diagnosable reasons that Section~\ref{sec:model} turns into a predictive
model: on \codeo{} the outcome is decided by a held-out test suite that never
appears in the trajectory, so no judge, trained or frontier, gets traction;
on \webm{}/\webf{} the base judge is already near-perfect and SFT saturates
what the budget allows; \tbench{} and \codes{} simply lack tasks.

\paragraph{Which reward direction matters, and where.}
Equation~(\ref{eq:reward}) has three natural arms: accuracy
($\lambda{=}\mu{=}1$), reject-weighted ($\mu{>}\lambda$), release-weighted
($\lambda{>}\mu$). On \agentsuite{} the direction is the effect: the
reject-weighted arm certifies $.321$ against $.297$ for accuracy and $.293$
(no gain over SFT) for release-weighted. Generic RL polish is not what moves
reject coverage there. On \arb{} all three arms land on the same $.585$
plateau: from a strong SFT start, any GRPO polish purifies the tail, and the
direction stops being separable. Both mechanisms are real; a practitioner
should try the reject-weighted arm first and expect it to matter most when the
SFT judge still has headroom (full table in Appendix~\ref{app:fulltables}).

\begin{table*}[t]
\centering
\caption{Strongly elicited frontier judges vs.\ the trained 4B judge, certified
by the same procedure. CoT: chain-of-thought then a verbalized probability;
SC-$k$: vote share over $k$ samples at $T{=}1$; logprob: native token
probability. Cost: relative inference cost per decision.}
\label{tab:frontier}
\small
\setlength{\tabcolsep}{4.5pt}
\begin{tabular}{@{}llcccc@{}}
\toprule
judge & elicitation & \multicolumn{2}{c}{\arb{}} & \multicolumn{2}{c}{\agentsuite{}} \\
\cmidrule(lr){3-4}\cmidrule(lr){5-6}
 & & AUROC & cert@.1 & AUROC & cert@.1 \\
\midrule
gpt-5.6-sol & CoT $+$ prob. & \textbf{.929} & \textbf{.704} & .794 & .000 \\
claude-sonnet-5 & CoT $+$ prob. & .905 & .497 & .847 & .000 \\
gemini-2.5-pro & CoT $+$ prob. & .753 & .000 & --- & --- \\
gemini-2.5-pro & SC-5 & .680 & .000 & --- & --- \\
claude-sonnet-5 & SC-10 & .686 & .000 & --- & --- \\
gpt-4o & logprob & .753 & .000 & --- & --- \\
\midrule
trained 4B (ours) & logprob & .922 & .585 & \textbf{.903} & \textbf{.297} \\
\bottomrule
\end{tabular}
\end{table*}

\paragraph{Frontier judges certify inconsistently, and recalibration cannot help.}
Table~\ref{tab:frontier} runs the elicitations a skeptical reviewer would
demand. The picture is not ``small beats frontier'' (the strongest reasoning
model beats our judge on \arb{}, $.704$ vs $.585$, at roughly $100\times$ the
inference cost) but that frontier judging does not \emph{transfer}: the same
model certifies \emph{zero} on \agentsuite{} at both budgets (AUROC $.794$,
matching our \emph{untrained} base), four of six frontier configurations
certify nothing anywhere, and only the trained 4B certifies on both corpora.
These are ranking failures, not score-scale failures, and not fixable
downstream:

\begin{proposition}[Rank invariance]
\label{prop:rank}
\boot{}'s certified coverage is invariant to any strictly increasing
transformation of the judge's scores. In particular, Platt scaling, isotonic
regression, and temperature scaling leave every certificate in
Table~\ref{tab:frontier} unchanged.
\end{proposition}

The proof is immediate: the grid is built from score quantiles and every
count in Eq.~(\ref{eq:booterr}) depends on scores only through comparisons
(Appendix~\ref{app:rank}). Still, the consequence is worth stating: a judge that
certifies $0$ cannot be rescued by recalibration, only by reordering, i.e.\ by
a different judge.

\paragraph{What the coverage buys.}
At six minutes of review per trajectory, Table~\ref{tab:main}'s operating
points remove roughly 30 reviewer-hours per thousand trajectories on
\agentsuite{} and 59 on \arb{}, with error $\le 0.1$ guaranteed on the removed
portion (Appendix~\ref{app:cost}); the judge runs on one local GPU, keeps
trajectories on-premise, and can grow its own coverage (\S\ref{sec:rsi}).

\begin{table*}[t]
\centering
\caption{Certificate-gated self-training (certified reject coverage @
$\alpha{=}0.1$); transfer judges never saw the target corpus.}
\label{tab:rsi}
\scriptsize
\setlength{\tabcolsep}{4pt}
\begin{tabular}{@{}llccl@{}}
\toprule
target corpus & judge & cert@.1 & harvest (contam.) & target labels \\
\midrule
\multirow{4}{*}{\arb{}}
 & transfer (round 0) & .575 & --- & 0 \\
 & \;$+$ \rsi{} round 1 & \textbf{.585} & 262/914 \;(.008) & 0 \\
 & 175 CAL labels spent on SFT instead & .540 & --- & 175 \\
 & in-domain ceiling (SFT$\to$GRPO, \S\ref{sec:judges}) & .585 & --- & 914 \\
\midrule
\multirow{2}{*}{\webm{}}
 & transfer (round 0) & \textbf{.830} & --- & 0 \\
 & \;$+$ \rsi{} round 1 & .820 & 837/1296 \;(.002) & 0 \\
\midrule
\multirow{3}{*}{\agentsuite{}}
 & in-domain SFT & .293 & --- & 2600 \\
 & \;$+$ \rsi{} round 1 & .343 & 296/2200 \;(.000) & $+0$ \\
 & \;$+$ \rsi{} round 2 & \textbf{.347} & 338/2200 \;(.003) & $+0$ \\
\bottomrule
\end{tabular}
\end{table*}

\section{Certified coverage is predictable before training}
\label{sec:model}

Seven corpora, three failure regimes, two successes: the pattern is regular
enough to model. Let $\pi$ be the corpus success rate and $A_0$ the AUROC of
the \emph{untrained} base judge, both measurable from a few hundred labeled
trajectories before any training run. $(1-\pi)$ is the reject side's raw
material and $(2A_0-1)$ rescales base discrimination; their product predicts
certified coverage remarkably well:
\begin{equation}
\label{eq:law}
\widehat{\covr}^\star \;=\; a\,(1-\pi)(2A_0 - 1) + b,
\end{equation}
with $a{=}1.08$, $b{=}{-}0.05$ fit across the seven corpora against the best
certified coverage each attains (Figure~\ref{fig:model}). Leave-one-corpus-out, the relation explains
$R^2 = 0.96$ of held-out coverage with mean absolute error $0.039$; in-sample
$R^2 = 0.98$. Both factors earn their place: dropping $(1-\pi)$ collapses
LOOCV $R^2$ to $0.38$, and AUROC alone reaches only $0.47$; an earlier version
carried a sample-size saturation factor that cross-validation rejects as
unnecessary (Appendix~\ref{app:model}). This is an empirical model, not a
law (seven points, one seed), but its practical content survives the caveat:
a team can score a few hundred trajectories with the \emph{base} judge and
read off, before spending a GPU-hour on training, roughly what fraction of
its evaluation queue is certifiably automatable.

\paragraph{Three gates for when RL helps.}
Whether GRPO adds coverage \emph{on top of SFT} follows from three conditions,
each visible in Table~\ref{tab:main}:
\begin{equation}
\label{eq:gates}
\underbrace{A_0 \gg 0.5}_{\text{judge can learn}} \,\wedge\,
\underbrace{A_{\mathrm{SFT}} < A_{\mathrm{sat}}}_{\text{SFT not saturated}}
\,\wedge\,
\underbrace{n_{\mathrm{eff}} \gtrsim n_{\min}(\alpha,\delta)}_{\text{enough
samples}},
\end{equation}
where $n_{\min}$ is the smallest effective sample a certificate can act on
(${\approx}29$ at $\alpha{=}0.1$, $\delta{=}0.05$: the zero-error
Clopper--Pearson point). \agentsuite{} and \arb{} pass all three gates and
show RL gains. \codeo{} fails the first: the outcome lives in a held-out test
suite, not in the trajectory, so the judge is blind and every training method
inherits the same $\approx 0.53$--$0.65$ ceiling, a property of the domain,
not the dataset: a second ingested code corpus ($6{,}306$ tasks) lands AUROC
in the same band. \webm{} and
\webf{} fail the second: base AUROC $0.983$ leaves SFT nothing to purify that
the budget can see, and no agent mix we tried (4B, 8B, two frontier models
generating fresh trajectories) changed the judgeability of the corpus.
\tbench{} ($n_{\mathrm{eff}}{=}30$) and \codes{} ($21$) fail the third:
whatever the judge does, the certificate cannot resolve improvements smaller
than its own finite-sample slack. The same effect appears in selection:
\agentsuite{}'s CAL certifies $0.150/0.149/0.145$ across the three GRPO
arms (too close to rank), so Table~\ref{tab:main} reports the CAL choice
bracketed by the best and worst arms. The gates are stated post-hoc; their
value is that each is measurable before training, and Section~\ref{sec:rsi}
uses them prospectively.

\section{The certificate as a self-training filter}
\label{sec:rsi}

Self-improving models share a weakness: the filter deciding which
self-generated labels to trust is heuristic, and when it drifts, contamination
compounds silently \citep{yuan2024self, huang2026r}. We already
have a non-heuristic filter: pseudo-labels harvested from a region certified
at budget $\alpha$ are, by definition of the certificate, wrong at rate at
most $\alpha$ (with probability $1-\delta$).
\rsi{} makes this operational (Algorithm~\ref{alg:rsi}, Appendix~\ref{app:rsi}): calibrate a certified
reject region on one half of CAL's tasks, pseudo-label its contents as
failures, retrain, re-certify: always on human-labeled CAL/TEST tasks that
pseudo-labels never touch, so the loop cannot erode its own guarantee.

\paragraph{The bound holds every time it is used.}
Across six harvests (three corpora, pools of $914$--$2{,}200$
trajectories), realized contamination against withheld ground truth was
$.041/.008/.024/.000/.003/.002$: never above the $0.10$ budget. The filter
is the deployment guarantee itself, not a proxy for it.

\paragraph{Entering a domain with zero training labels.}
Trained leave-one-corpus-out, with no web data at all, a judge already
certifies $0.560$ on \arb{}; one \rsi{}
round closes the gap to the in-domain \emph{GRPO} ceiling ($.585$) with zero
\arb{} training labels, and the same judge certifies $0.830$ on \webm{},
matching that ceiling outright (Table~\ref{tab:rsi}). Two controls: the same
175 CAL labels spent on supervision instead certify only $0.540$; where
transfer is weak (\agentsuite{}, AUROC $.794$), \rsi{} \emph{refuses to
harvest}: no certificate, no self-training.

\paragraph{In-domain growth, and where iteration ends.}
From the in-domain SFT judge on \agentsuite{}, certified harvesting from the
unseen RL pool lifts coverage to $.343$ in one round and $.347$ in
two, past the best supervised GRPO arm ($.321$), with no labels beyond
SFT's. Iteration is not free: on \arb{} a second round degrades
coverage ($.585\to.465$) even though its harvest stayed within budget
($.024$). The $\alpha$-bound controls how many pseudo-labels are \emph{wrong},
not how \emph{skewed} the harvested distribution is; a second pass over a
small one-sided pool amplifies its selection bias, and CAL cannot rank
rounds on 175 rows (Appendix~\ref{app:rsi}); so one round for transfer
starts and small pools, iteration only in-domain on large pools. On
saturated \webm{} the round is correctly useless (Table~\ref{tab:rsi}), as
Eq.~(\ref{eq:gates}) predicts.

\section{Conclusion}
\label{sec:conclusion}

Agent-evaluation judges deserve a better answer than an agreement rate: a
certificate, valid under task clustering, for the evaluation work a judge can
provably take over.

\section*{Limitations}
Every number in this paper comes from a single pre-registered run: task-level
splits fixed by one seed, hyperparameters chosen on calibration data, one test
read per configuration. The protocol prevents selection effects but does not
quantify run-to-run variance; the cluster bootstrap quantifies sampling
uncertainty over tasks, not over training randomness. \boot{}'s guarantee is
asymptotic in the number of task clusters: we validate it down to 20 clusters
by simulation and audit it out-of-sample on every corpus, but it is not a
finite-sample theorem, and the anti-conservativeness of the naive i.i.d.\
certificate is demonstrated on one real corpus in the high-clustering
multi-rollout regime rather than universally. The certifiability model is fit
on seven corpora and should be read as an empirical trend. On our corpora the
release side rarely certifies at practical budgets, so the deployed guarantee
mostly removes review of \emph{failures}; and certified self-training is a
one-round recommendation. Iterating on small one-sided pools degraded
coverage even with the contamination bound intact.

\bibliography{refs}

\clearpage
\appendix
\section{Rank invariance of the certificate}
\label{app:rank}

\begin{proof}[Proof of Proposition~\ref{prop:rank}]
Let $\phi:[0,1]\to[0,1]$ be strictly increasing and $\tilde s = \phi\circ s$.
The grid $\Theta$ consists of empirical quantiles of the scores, so
$\tilde\Theta = \phi(\Theta)$ elementwise, and for every $\theta\in\Theta$ and
every trajectory, $\tilde s(x) \le \phi(\theta) \iff s(x)\le\theta$. All
quantities in Algorithm~\ref{alg:boot} (the per-task counts $n_g(\theta),
k_g(\theta)$, every bootstrap ratio $\widehat{\err}^{(b)}(\theta)$, the
quantile test, and the coverage of the selected threshold) depend on scores
only through such comparisons, hence are identical under $\phi$. The same
argument applies to the naive and DEFF-corrected certificates, whose grids and
counts are constructed the same way. Platt scaling, temperature scaling, and
isotonic regression (with ties broken consistently) are monotone, which gives
the statement in the main text.
\end{proof}

Two practical corollaries. First, the frontier failures in
Table~\ref{tab:frontier} cannot be repaired by post-hoc calibration on our 175
CAL labels or any amount of it; the orderings themselves are wrong. Second,
the judge's absolute probability calibration is irrelevant to certification;
what Eq.~(\ref{eq:score}) must get right is the \emph{ranking} of failures
below successes, which is what SFT and the GRPO arms improve. Throughout, the
exact binomial bound is Clopper--Pearson
\citep{be7c0fd0-f562-39ad-b8e0-716a276561d1} and the resampling principle is
Efron's \citep{10.1214/aos/1176344552}, applied at the cluster level
\citep{field2007bootstrapping}.

\section{Synthetic validity study}
\label{app:sim}

\paragraph{Data-generating process.}
Task effects $u_g \sim \mathcal N(0, \tau^2)$; cluster sizes $m_g \sim
1+\mathrm{Poisson}(7)$; labels $y_{gj} \sim \mathrm{Bernoulli}\big(\sigma(b_0
+ u_g)\big)$; scores $s_{gj} = \sigma\big(d\,(2y_{gj}-1) + c\,u_g +
\varepsilon_{gj}\big)$, $\varepsilon_{gj}\sim\mathcal N(0,1)$, with $(b_0, d,
c) = (-0.8,\, 2.2,\, 0.9)$ and $\tau$ set to $\{0.9, 2.6, 5.0\}$ to hit label
ICC $\rho \approx \{0.1, 0.5, 0.8\}$. Each of 300 trials draws a calibration
set of $G$ tasks, runs all five procedures ($B{=}800$, $|\Theta|{=}40$,
$\alpha{=}0.1$, $\delta{=}0.05$), and evaluates the selected threshold's true
selective error on a fresh draw of $1.5\times10^5$ trajectories. Violation =
fraction of trials with true error $>\alpha$, computed among trials where the
procedure certified anything.

\begin{table*}[t]
\centering
\caption{Full synthetic grid: violation / mean certified coverage.}
\label{tab:simfull}
\scriptsize
\setlength{\tabcolsep}{3.5pt}
\begin{tabular}{@{}llccccc@{}}
\toprule
$G$ & $\rho$ & naive CP & DEFF-CP & one-per-task & task-Hoeffding & \boot{} \\
\midrule
20 & .1 & .00 / .65 & .00 / .16 & .00 / .00 & .00 / .00 & .00 / .68 \\
20 & .5 & .00 / .53 & .00 / .00 & .00 / .00 & .00 / .00 & .01 / .61 \\
20 & .8 & .00 / .47 & .00 / .00 & .00 / .00 & .00 / .00 & .01 / .56 \\
50 & .1 & .00 / .69 & .00 / .67 & .00 / .00 & .00 / .00 & .00 / .69 \\
50 & .5 & .00 / .61 & .00 / .03 & .00 / .00 & .00 / .00 & .00 / .62 \\
50 & .8 & .00 / .57 & .00 / .00 & .00 / .00 & .00 / .00 & .00 / .57 \\
100 & .1 & .00 / .70 & .00 / .69 & .00 / .29 & .00 / .00 & .00 / .70 \\
100 & .5 & .00 / .62 & .00 / .58 & .03 / .06 & .00 / .00 & .00 / .63 \\
100 & .8 & .00 / .58 & .00 / .23 & .00 / .00 & .00 / .00 & .00 / .58 \\
500 & .1 & .00 / .71 & .00 / .71 & .00 / .69 & .00 / .66 & .00 / .71 \\
500 & .5 & .00 / .64 & .00 / .63 & .00 / .62 & .00 / .56 & .00 / .64 \\
500 & .8 & .00 / .60 & .00 / .58 & .00 / .58 & .00 / .00 & .00 / .59 \\
\bottomrule
\end{tabular}
\end{table*}

\paragraph{Adversarial designs.}
Three additional designs target the regimes where a reviewer would expect the
naive certificate to break at population level: (i) parameters matched to
\webm{} ($G{=}32$, $\rho{\approx}0.8$, near-perfect discrimination
$d{=}4.5$, low base rate); (ii) size--outcome correlation, $u_g \gets u_g -
0.25\,(m_g - \bar m)$, so large clusters fail more; (iii) heavy-tailed cluster
sizes, $m_g \sim 1{+}\min(\lfloor 4\,\mathrm{Pareto}(1.3)\rfloor{+}1, 60)$.
Naive violation rates were $.003$, $.000$, $.000$ respectively (\boot{}:
$.000$ in all three, equal coverage). We report this against ourselves: the
Bonferroni-slack-protected naive certificate is population-valid in every
synthetic regime we constructed honestly, and its demonstrated failure is the
task-resampling audit on real high-$\rho$ data (Appendix~\ref{app:audits}).
The two statements are compatible: the audit measures the dispersion of
realized error on corpora like the observed one, which is the quantity a
deployment re-run experiences. Together they motivate the uniform-safety
framing of Section~\ref{sec:cert} rather than a blanket invalidity claim.

\paragraph{Monte-Carlo error.}
With 300 trials, a true violation rate of $0.05$ has standard error
$\approx0.013$; the \boot{} estimates of $\le 0.01$ are consistent with a true
rate at or below the budget in every cell.

\section{Cluster-valid baselines on the real corpora}
\label{app:baselines}

\begin{table*}[t]
\centering
\caption{Certified reject coverage @ $\alpha{=}0.1$ on real test scores, five
procedures. One-per-task uses a single pre-registered draw (seed 42);
task-Hoeffding tests the mean of per-task error rates over covered tasks
(estimand: task-weighted error), Bonferroni $\delta/40$ throughout.}
\label{tab:realbase}
\scriptsize
\setlength{\tabcolsep}{4pt}
\begin{tabular}{@{}llrrccccc@{}}
\toprule
corpus & judge & $G$ & $\rho$ & naive & DEFF & one/task & Hoeffding & \boot{} \\
\midrule
\agentsuite{} & base & 84 & .25 & .000 & .000 & .000 & .000 & .000 \\
\agentsuite{} & SFT & 84 & .25 & .325 & .000 & .000 & .000 & .325 \\
\agentsuite{} & GRPO$_{\mu3}$ & 84 & .25 & .321 & .000 & .000 & .000 & .321 \\
\tbench{} & base & 13 & .43 & .340 & .000 & .000 & .000 & .315 \\
\tbench{} & SFT & 13 & .43 & .379 & .000 & .000 & .000 & .379 \\
\tbench{} & GRPO$_{\mu5}$ & 13 & .43 & .369 & .000 & .000 & .000 & .369 \\
\arb{} & base & 68 & .49 & .465 & .000 & .000 & .000 & .510 \\
\arb{} & SFT & 68 & .49 & .465 & .000 & .000 & .000 & .560 \\
\arb{} & GRPO$_{\mu3}$ & 68 & .49 & .430 & .000 & .000 & .000 & .585 \\
\webm{} & base & 32 & .80 & .857 & .000 & .000 & .000 & .832 \\
\webm{} & SFT & 32 & .80 & .859 & .000 & .000 & .000 & .789 \\
\webf{} & base & 32 & .81 & .707 & .000 & .000 & .000 & .707 \\
\webf{} & SFT & 32 & .81 & .758 & .000 & .000 & .000 & .758 \\
\codeo{} & base & 382 & .64 & .000 & .000 & .000 & .000 & .000 \\
\codes{} & base & 7 & .00 & .000 & .000 & .000 & .000 & .619 \\
\bottomrule
\end{tabular}
\end{table*}

The two finite-sample-valid constructions certify zero on all fifteen rows,
including \codeo{} with $382$ tasks (its judge is too weak) and \agentsuite{}
with $84$ (the $\delta/40$-corrected exact bounds need more). The \codes{} row
illustrates the $G<20$ caveat from the main text: \boot{} nominally certifies
$.619$ from seven clusters, outside the regime our simulation validates, and
we do not use that number anywhere.

\section{Certificate audits}
\label{app:audits}

\paragraph{Task-resampling audit of the naive certificate.}
For each corpus and judge, the naive i.i.d.\ certificate selects its
maximal-coverage threshold; we then resample tasks with replacement
($3{,}000$ draws) and record how often the realized selective error at that
threshold exceeds $\alpha$. Values $\le\delta{=}.05$ are consistent with the
promised confidence:
\webm{} SFT: $.069$ ($\alpha{=}.1$), $.175$ ($\alpha{=}.2$);
\webm{} base: $.017$, $.021$;
all other corpus--judge pairs $\le .032$
(\agentsuite{} SFT $.004/.032$, GRPO $.001/.017$; \arb{} all $\le.004$;
\tbench{} SFT $.000/.027$; \webf{} $\le.016$).
The naive certificate's failure is confined to, and severe in, the high-$\rho$
multi-rollout regime, where it also claims the most ($.98$ coverage at
$\alpha{=}.2$).

\paragraph{CAL$\to$TEST audits of \boot{}.}
Thresholds calibrated on CAL tasks, realized selective error measured once on
task-disjoint TEST, $\alpha=0.1$:
\arb{} SFT $.028$, GRPO$_{\mu3}$ $.027$, base $.037$;
\agentsuite{} SFT $.012$, GRPO$_{\mu1}$ $.012$, GRPO$_{\mu3}$ $.011$;
\webm{} base $.008$, SFT $.015$, GRPO$_{\mu5}$ $.008$;
\webf{} base $.016$, SFT $.022$, GRPO$_{\mu3}$ $.016$.
At $\alpha=0.2$ all audits likewise pass (maximum realized error $.116$,
on \arb{} GRPO$_{\mu3}$). Every deployed threshold in the paper comes from
this protocol.

\section{Full result tables}
\label{app:fulltables}

\paragraph{All GRPO arms, both headline corpora} (certified reject coverage,
$\alpha{=}0.1$; CAL-certified coverage used for selection in parentheses):

\begin{table*}[t]
\centering
\scriptsize
\setlength{\tabcolsep}{5pt}
\begin{tabular}{@{}lcccccc@{}}
\toprule
corpus & SFT & acc.\ ($\lambda{=}\mu{=}1$) & reject $\mu{=}3$ & reject $\mu{=}5$ & release $\lambda{=}5$ & CAL pick \\
\midrule
\agentsuite{} & .293 & .297 (.150) & .321 (.149) & .295 (.145) & .293 & $\mu{=}1$ \\
\arb{} & .560 & .585 (.514) & .585 (.537) & .585 (.514) & .585 & $\mu{=}3$ \\
\bottomrule
\end{tabular}
\end{table*}

\paragraph{$\alpha=0.2$, reject side (\boot{})}:
\arb{}: base $.635$, SFT $.710$, GRPO$_{\mu1}$ $.735$, GRPO$_{\mu3}$ $.710$,
GRPO$_{\mu5}$ $.715$;
\agentsuite{}: SFT $.422$, GRPO$_{\mu1}$ (CAL pick) $.424$, GRPO$_{\mu3}$
$.391$, base $.120$;
\webm{}: base $.883$, SFT $.859$, GRPO$_{\mu5}$ $.885$;
\webf{}: base $.707$, SFT $.758$, GRPO$_{\mu3}$ $.758$;
\tbench{} (grey regime): base $.442$, SFT $.469$, GRPO$_{\mu5}$ $.423$.
Frontier at $\alpha{=}0.2$ on \arb{}: gpt-5.6-sol CoT $.784$, sonnet-5 CoT
$.675$, all other configurations $.000$; on \agentsuite{}: sonnet-5 CoT
$.124$, gpt-5.6-sol CoT $.000$.

\paragraph{Release side.}
The release budget certifies far less everywhere: the only non-zero cells at
$\alpha{=}0.2$ are on \agentsuite{} (SFT $.27$) and nothing certifies at
$\alpha{=}0.1$ on any corpus. We therefore report the two-sided formulation as
framework and validate the reject side; release-side validation at realistic
budgets needs corpora with more certifiable success mass than ours have.

\paragraph{Frontier parse rates.}
CoT with verbalized probability parsed $199/200$ (gpt-5.6-sol), $182/200$
(gemini-2.5-pro), $191/200$ (sonnet-5) on \arb{}; $1194/1280$ and $741/1280$
for gpt-5.6-sol and sonnet-5 on \agentsuite{}; SC-$k$ parsed $\ge 199/200$;
gpt-5.2's API returns no token logprobs, so the frontier-logprob row uses
gpt-4o ($200/200$). Unparsed rows are excluded from that judge's scores (not
counted against it).

\section{Corpora}
\label{app:corpora}

All corpora share one schema (task id, rendered trajectory text, binary
outcome) and one split procedure: tasks are shuffled with a fixed seed and
assigned $40/30/15/15$ percent to SFT/RL/CAL/TEST (test fraction raised to
$25\%$ on the MiniWoB corpora to clear $n_{\min}$); trajectories follow their
task, so all four splits are task-disjoint, verified programmatically.
Rendered text contains the task instruction and per-step actions and
observation excerpts, middle-truncated to a fixed character budget; no reward
signal, evaluator output, or environment verdict is ever rendered.

\textbf{\agentsuite{}} ($\tau^2$-bench;
\citealp{barres2025tau2benchevaluatingconversationalagents}): tool-use
dialogues across airline, retail, and telecom domains; outcomes from the
benchmark's database-state checks. $6{,}400$ trajectories over $421$ tasks
($1{,}280$/$84$ in TEST).
\textbf{\tbench{}} \citep{merrill2026terminalbenchbenchmarkingagentshard}:
terminal-session tasks with programmatic goal checks.
\textbf{\arb{}}: the AgentRewardBench release
\citep{lu2025agentrewardbench}, $1{,}302$ trajectories of four agent models
on five live web benchmarks (WebArena, VisualWebArena, AssistantBench,
WorkArena, WorkArena++; \citealp{zhou2024webarena, koh2024visualwebarena,
yoran2024assistantbench, drouin2024workarena, boisvert2407workarena++}), each
expert-annotated; we use the $1{,}289$ trajectories with a majority success
label. Cluster = one web task, pooling the agents that attempted it.
\textbf{\webm{}}: MiniWoB++ trajectories
\citep{liu2018reinforcementlearningwebinterfaces} we generate under
BrowserGym \citep{dechezelles2025browsergymecosystemwebagent} with two weak
agents (a 4B vision--language model and its SFT variant) across all $125$
registered tasks, eight seeds each, labels from the environment reward; the
trajectory text excludes the reward. \textbf{\webf{}}: same environments,
trajectories from two frontier agents (a reasoning model and a fast
commercial model), six/two seeds.
\textbf{\codeo{}}: code-repair trajectories from SWE-rebench
\citep{NEURIPS2025_21bec6ac} under the OpenHands scaffold
\citep{ICLR2025_a4b6ad6b}, whose outcome is a held-out test suite in the
SWE-bench style \citep{ICLR2024_edac78c3}; \textbf{\codes{}}: a small
multilingual corpus from Multi-SWE-bench \citep{NEURIPS2025_5afa9cb1}. Full
provenance, counts, and rendering parameters: Table~\ref{tab:corpora} and the
released scripts.

\section{Judge prompt}
\label{app:prompts}

One prompt is shared verbatim by base scoring, SFT supervision, GRPO rollouts,
and all trained-judge scoring; training and evaluation cannot drift apart.

\begin{quote}\small\ttfamily
\textbf{system:} You are a skeptical auditor of an AI agent's trajectory.
Decide whether the agent truly completed its task. The agent's own claims of
success are not evidence; judge only from what the trajectory shows.

\medskip
\textbf{user:} Task instruction:\\
"""$\langle$task$\rangle$"""\\[2pt]
Agent trajectory (actions and observation excerpts, possibly truncated):\\
$\langle$rendered trajectory$\rangle$\\[2pt]
Did the agent truly and completely succeed at the task? Answer with one word,
SUCCESS or FAIL.\\
VERDICT:
\end{quote}

The score (Eq.~\ref{eq:score}) reads the next-token distribution after
\texttt{VERDICT:}; both verdict words are single tokens for our base model.
Frontier CoT elicitation replaces the one-word instruction with step-by-step
reasoning ending in \texttt{P\_SUCCESS=<integer 0-100>}; SC-$k$ keeps the
one-word form at temperature $1.0$.

\section{Training details}
\label{app:training}

\paragraph{Judge.} Qwen3.5-4B base \citep{yang2025qwen3technicalreport}, LoRA
\citep{hu2021loralowrankadaptationlarge} on all attention and MLP projections
($32.5$M trainable parameters, $0.77\%$), reasoning disabled at train and test
time. The \webm{} generating agents are Qwen3-VL models
\citep{bai2025qwen3vltechnicalreport}. \textbf{SFT}: one epoch over the corpus's SFT split, lr $2\times10^{-5}$
cosine, batch 2 with gradient accumulation 4, max sequence $5{,}120$ tokens,
completion-loss only (the single verdict token). \textbf{GRPO}: initialized
from the SFT adapter, reward per Eq.~(\ref{eq:reward}), 8 samples per prompt,
lr $2\times10^{-6}$, max completion 40 tokens, 150 optimizer steps (40 on
\tbench{}), batch 8 per device on 2--6 GPUs (DDP; identical results at
different world sizes, as expected for a fixed global batch). \textbf{RSI
rounds}: identical SFT recipe over source data plus harvested pseudo-labels.
All runs use seed 42. Hardware: one node with 96\,GB GPUs; a full
SFT$+$GRPO$+$certification pass for one corpus takes $1$--$3$ GPU-hours; every
certificate and analysis in the paper runs on CPU in minutes.

\section{Leave-one-corpus-out transfer}
\label{app:loco}

Transfer judges are SFT-trained on a balanced sample ($\le 2{,}600$ rows,
equal per source benchmark) of the pooled SFT splits of all corpora
\emph{except} the held-out one, then certified on the held-out corpus's
untouched TEST via CAL-calibrated \boot{}. Two variants for \arb{}: excluding
only \arb{} itself (other web corpora remain; certifies $.575$, AUROC $.915$)
and the strict variant excluding \emph{all} web corpora (certifies $.560$,
AUROC $.902$), the number quoted as ``zero web exposure.'' The reverse
direction is the honest negative: a judge trained on everything except
$\tau^2$ transfers at AUROC $.794$ and certifies $.045$; tool-use dialogue
formats are idiosyncratic in a way generic trajectory-reading does not cover,
which is precisely the situation the \rsi{} gate then refuses to harvest in.
The \webm{} transfer row in Table~\ref{tab:rsi} uses the strict
zero-web-exposure judge.

\section{\rsi{} protocols and the stopping-rule negative}
\label{app:rsi}

\begin{algorithmwide}[t]
\caption{\rsi{} (one round, reject side)}
\label{alg:rsi}
\begin{algorithmic}[1]
\Require judge $s_0$; unlabeled pool $\mathcal U$; CAL split by task into halves $(\mathcal C_A, \mathcal C_B)$; $\alpha$
\State $\hat\theta \gets \boot{}(s_0, \mathcal C_A, \alpha)$
\Comment{calibrate harvest region on half the CAL tasks}
\State $\mathcal H \gets \{x \in \mathcal U : s_0(x) \le \hat\theta\}$, pseudo-labeled $\tilde y = 0$
\Comment{$\Pr(y{=}1\,|\,x\in\mathcal H) \le \alpha$ w.p.\ $1{-}\delta$}
\State $s_1 \gets$ fine-tune on source data $\cup\, \mathcal H$;
\Return $s_1$, certified on $\mathcal C_B$, audited on TEST
\Comment{pseudo-labels never enter CAL/TEST}
\end{algorithmic}
\end{algorithmwide}

\paragraph{Protocols.} Pool = the target corpus's SFT$+$RL splits with labels
withheld (transfer starts) or its unseen RL split (in-domain starts); harvest
threshold from \boot{} on half of CAL's tasks (seed-42 halves); round-$k$
training set = source data $+$ current harvest; certification of every round
on the untouched other half and TEST. Ground-truth labels of harvested
trajectories are used only to \emph{report} contamination, never in training.

\begin{table*}[t]
\centering
\small
\begin{tabular}{@{}llccc@{}}
\toprule
target & start & harvest & contamination & TEST cert @ .1 \\
\midrule
\arb{} & transfer, round 1 (full-CAL pilot) & 437/914 & .041 & .585 \\
\arb{} & transfer, round 1 (split-CAL) & 262/914 & .008 & .585 \\
\arb{} & transfer, round 2 & 336/914 & .024 & .465 \\
\agentsuite{} & in-domain, round 1 & 296/2200 & .000 & .343 \\
\agentsuite{} & in-domain, round 2 & 338/2200 & .003 & .347 \\
\webm{} & transfer, round 1 & 837/1296 & .002 & .820 \\
\agentsuite{} & transfer & \multicolumn{3}{c}{gate closed: CAL certifies nothing} \\
\bottomrule
\end{tabular}
\caption{All \rsi{} harvests. The full-CAL pilot row shows the protocol
before the split-CAL fix; its result is unchanged by the fix, and all
reported numbers use the split-CAL protocol.}
\label{tab:allharvests}
\end{table*}

\paragraph{All harvests.}
Table~\ref{tab:allharvests} lists every harvest, including the full-CAL
pilot run before the split-CAL hygiene fix.

\paragraph{Why not a stopping rule.}
A natural iteration rule (continue while CAL-certified coverage
improves) fails empirically: on \arb{}, CAL certifies $.514 / .389 / .486$
for rounds $0/1/2$ while TEST moves $.575 / .585 / .465$; a 175-row CAL cannot
rank models this close (the same resolution limit as arm selection on
\agentsuite{}). We therefore recommend the fixed policy of
Section~\ref{sec:rsi} rather than an adaptive rule the calibration data
cannot support.

\section{Certifiability model: validation}
\label{app:model}

Table~\ref{tab:lawval} reports leave-one-corpus-out validation of
Eq.~(\ref{eq:law}) and ablated forms (target: best certified coverage per
corpus).

\begin{table*}[t]
\centering
\small
\begin{tabular}{@{}lccc@{}}
\toprule
form & in-sample $R^2$ & LOOCV $R^2$ & LOOCV MAE \\
\midrule
$(1-\pi)(2A_0-1)$ \hfill (Eq.~\ref{eq:law}) & .977 & .964 & .039 \\
$(1-\pi)(2A_0-1)(1-e^{-n_{\mathrm{eff}}/K})$, $K{=}30$ & .835 & .707 & .124 \\
\quad same, $K$ tuned by LOOCV ($K{=}10$) & .977 & .967 & .039 \\
$(2A_0-1)$ only & .913 & .471 & .132 \\
$(2A_0-1)(1-e^{-n_{\mathrm{eff}}/30})$ (no $\pi$) & .656 & .382 & .196 \\
\bottomrule
\end{tabular}
\caption{Certifiability-model validation: in-sample and leave-one-corpus-out
fit of Eq.~(\ref{eq:law}) and ablated forms.}
\label{tab:lawval}
\end{table*}

The saturation factor adds nothing once $K$ is tuned and hurts when fixed;
the base-rate factor is load-bearing. Caveats stated in the main text apply:
$n{=}7$ corpora, and the zero-coverage corpus anchors the low end of the fit.

\section{Cost accounting}
\label{app:cost}

Certified coverage converts to removed review time linearly: at $t$ minutes
of human review per trajectory, a corpus with certified coverage $c$ saves
$1000\,c\,t/60$ reviewer-hours per thousand trajectories, with residual risk
bounded by $\alpha$ on the removed portion. At $t{=}6$:
\agentsuite{} $29.7$\,h, \arb{} $58.5$\,h, \webm{} $83.6$\,h, \webf{}
$75.8$\,h per thousand. Judge-side marginal cost is one forward pass of a 4B
model per trajectory (batched, a single GPU sustains ${\sim}10$
trajectories/second at our sequence lengths); the strongest frontier
configuration spends ${\sim}10^3$ reasoning tokens per trajectory at API
prices, the basis for the ${\sim}100\times$ figure in
Section~\ref{sec:judges}. We keep all dollar figures in footnotes because
they inherit local prices; the hour figures do not.

\end{document}